%% file: arxiv_head.tex
\documentclass[11pt]{article}

\usepackage[vmargin=1in,hmargin=1in,centering,letterpaper]{geometry}
\usepackage{times}
\usepackage[utf8]{inputenc} % allow utf-8 input
\usepackage[T1]{fontenc}    % use 8-bit T1 fonts
\usepackage{hyperref}       % hyperlinks
\usepackage{url}            % simple URL typesetting
\usepackage{booktabs}       % professional-quality tables
\usepackage{amsfonts}       % blackboard math symbols
\usepackage{nicefrac}       % compact symbols for 1/2, etc.
\usepackage{microtype}      % microtypography
\usepackage{dsfont}
\usepackage{bm}
\usepackage{color}
\usepackage{graphicx}

\usepackage{amssymb} %maths
\usepackage{amsmath} %maths
\usepackage{amsthm}

\begin{document}
\include{tripletsmds}
\end{document}

%% file: tripletsmds.tex
%!TEX root = arxiv_head.tex

\newtheorem{theorem}{Theorem}
\newtheorem{lemma}{Lemma}
\newtheorem{assume}{Assumption}

\def \E{\mathbb E}
\def \P{\mathbb P}
\def \1{\mathds{1}}
\def \R{\mathbb R}
\def \bx{\boldsymbol{x}}
\def \bs{\boldsymbol{s}}
\def \bv{\boldsymbol{v}}
\def \bz{\boldsymbol{z}}
\def \be{\boldsymbol{e}}
\def \bX{\boldsymbol{X}}
\def \bT{\boldsymbol{T}}
\def \bD{\boldsymbol{D}}
\def \bB{\boldsymbol{B}}
\def \bE{\boldsymbol{E}}
\def \bM{\boldsymbol{M}}
\def \bG{\boldsymbol{G}}
\def \bC{\boldsymbol{C}}
\def \bV{\boldsymbol{V}}
\def \bU{\boldsymbol{U}}
\def \bS{\boldsymbol{S}}
\def \bI{\boldsymbol{I}}
\def \bJ{\boldsymbol{J}}
\def \bL{\boldsymbol{L}}
\def \bL{\boldsymbol{L}}
\def \bZ{\boldsymbol{Z}}
\def \b1{\boldsymbol{1}}
\def \T{{\cal T}}
\def \G{{\cal G}}
\def \S{{\cal S}}
\def \tr{\text{Tr}}
\def\L{\mathcal{L}}
\def\imwidth{\textwidth}
\newcommand{\kevin}[1]{{\color{red}{#1}}}

\newcommand{\mS}{\mathbb{S}}
\newcommand{\rank}{\mbox{rank}}
\newcommand{\diag}{\mbox{diag}}
\newcommand{\pg}{\langle \bL_t, \bG \rangle}
\newcommand{\pt}{\langle \bL_t, \bG^\star \rangle}

\title{Finite Sample Prediction and Recovery Bounds for Ordinal Embedding}

\author{
  Lalit Jain, University of Wisconsin-Madison\\
  Kevin Jamieson, University of California, Berkeley\\
  Robert Nowak, University of Wisconsin-Madison
}

\maketitle
\begin{abstract}
The goal of ordinal embedding is to represent items as points in a low-dimensional Euclidean space given a set of constraints in the form of distance comparisons like ``item $i$ is closer to item $j$ than item $k$''.  Ordinal constraints like this often come from human judgments.  To account for errors and variation in judgments, we consider the noisy situation in which the given constraints are independently corrupted by reversing the correct constraint with some probability. This paper makes several new contributions to this problem. First, we derive prediction error bounds for ordinal embedding with noise by exploiting the fact that the rank of a distance matrix of points in $\R^d$ is at most $d+2$. These bounds characterize how well a learned embedding predicts new comparative judgments.  Second, we investigate the special case of a known noise model and study the Maximum Likelihood estimator.  Third, knowledge of the noise model enables us to relate prediction errors to embedding accuracy.  This relationship is highly non-trivial since we show that the linear map corresponding to distance comparisons is non-invertible, but there exists a nonlinear map that is invertible.  Fourth, two new algorithms for ordinal embedding are proposed and evaluated in experiments.
\end{abstract}

\section{Ordinal Embedding}

Ordinal embedding, also known as non-metric multidimensional scaling, aims to represent items as points in $\R^d$ so that the distances between items agree as well as possible with a given set of ordinal comparisons such as item $i$ is closer to item $j$ than to item $k$.  This is a classic problem that is often used to visualize perceptual similarities \cite{shepard1962analysis,kruskal1964nonmetric}.
Recently, several authors have proposed a variety of new algorithms for learning a metric embedding from ordinal information \cite{agarwal2007generalized,tamuz2011adaptively,jamieson2011low,van2012stochastic,mcfee2011learning}. There has also been some theoretical progress towards characterizing the consistency of ordinal embedding methods. For example, it has been shown that the correct embedding can be learned in the limit as the number of items grows \cite{kleindessner2014uniqueness,terada2014local,arias2015some}.  However, a major shortcoming of all prior work is the lack of generalization and embedding error bounds for problems involving a finite number of items and observations.  This paper addresses this problem, developing error bounds for ordinal embedding algorithms with known observation models.  The bounds explicitly show the dependence on the number of ordinal comparisons relative to the number of items and the dimension of the embedding. In addition, we propose two new algorithms for recoverying embeddings: the first is based on unbiasing a nuclear norm constrained optimization and the second is based on projected gradient descent onto the space of rank-$d$ matrices, commonly referred to as hard-thresholding. Both methods match state-of-the-art performance while being simpler to implement and faster to converge.

\subsection{Ordinal Embedding from Noisy Data}

Consider $n$ points $\bx_1,\bx_2,\dots,\bx_n \in \R^d$.  Let $\bX =
[\bx_1 \cdots \bx_n] \in \R^{d \times n}$.  The {\em Euclidean distance matrix} $\bD^\star$ is defined to have elements $D_{ij}^\star = \|\bx_i-\bx_j\|_2^2$.  
Ordinal embedding is the problem of recovering $\bX$ given ordinal constraints on distances.  This paper focuses on ``triplet'' constraints of the form $D_{ij}^\star<D_{ik}^\star$, where $1 \leq i \neq j \neq k \leq n$.
Furthermore, we only observe noisy indications of these constraints, as follows. Each triplet $t=(i,j,k)$ has an associated probability $p_t$ satisfying
\[
p_t > 1/2 \ \Longleftrightarrow \ \|\bx_i-\bx_j\|^2 < \|\bx_i-\bx_k\|^2 \ .
\]
Let $\S$ denote a collection of triplets drawn independently and uniformly at random. And for each $t\in \S$  we observe an independent random variable $y_t =-1$ with probability $p_t$, and $y_t =1$ otherwise.
The goal is to recover the embedding $\bX$ from these data.
Exact recovery of $\bD^\star$ from such data requires a known {\em link} between $p_t$ and $\bD^\star$.  To this end, our main focus is the following problem. \\

\centerline{\fbox{\parbox[b]{\textwidth}{\underline{\bf Ordinal Embedding from Noisy Data} \\ \vspace{-.05in} \\
Consider $n$ points $\bx_1, \bx_2 \cdots,\bx_n$ in $d$-dimensional Euclidean space. 
Let $\S$ denote a collection of triplets and for each $t\in \S$  observe an independent random variable
$$y_t \ = \  \left\{\begin{array}{ll} -1 & w.p.\ f(D_{ij}^\star-D_{ik}^\star) \\ \ \\
1 & w.p.\  1-f(D_{ij}^\star-D_{ik}^\star) \end{array}\right. \ . $$
where the link function $f:\R \rightarrow [0,1]$ is known.
Estimate $\bX$ from $\S$, $\{y_t\}$, and $f$. }}}

For example, if $f$ is the logistic function, then for triplet $t=(i,j,k)$
\begin{align}\label{eqn:logistic_link}
p_t=\P(y_t = -1 ) \ = \ f(D_{ij}^\star-D_{ik}^\star) \ = \ \frac{1}{1+\exp(D_{ij}^\star-D_{ik}^\star)} \ , 
\end{align}
then $D_{ij}^\star-D_{ik}^\star = \log\big(\tfrac{1-p_t}{p_t}\big)$. However, we stress that we only require the existence of a link function for exact recovery of $\bD^\star$. Indeed, if one just wishes to {\em predict} the answers to unobserved triplets, then the results of Section~\ref{predict} hold for arbitrary $p_t$ probabilities. We note that a problem known as one-bit matrix completion has similarly used link functions to recover structure \cite{davenport20141}. However, while that work made direct measurements of the entries of the matrix, in this work we observe linear projections and as we will see later, the collection of these linear operators has a a non-empty kernel creating non-trivial challenges.

\subsection{Organization of Paper}

This paper takes the following approach to ordinal embedding.  First we derive prediction error bounds for ordinal embedding with noise in Section~\ref{predict}.  These bounds exploit the fact that the rank of a distance matrix of points in $\R^d$ is at most $d+2$. Then we consider the special case of a known observation model and the Maximum Likelihood estimator in Section~\ref{mle_sec}.  The link function enables us to relate prediction errors to error bounds on estimates of differences $\{D_{ij}^\star-D_{ik}^\star\}$.  In Section~\ref{embedding}, we study the non-trivial problem of recovering $\bD^\star$, and thus $\bX$, from the $\{D_{ij}^\star-D_{ik}^\star\}$. Lastly, in Section~\ref{exp}, two new algorithms for ordinal embedding are proposed and experimentally evaluated.

\subsection{Notation and Assumptions}
We will use $(\bD^\star,\bG^\star)$ to denote the distance and Gram matrices of the latent embedding, and $(\bD,\bG)$ to denote an arbitrary distance matrix and its corresponding Gram matrix.
The observations $\{y_t\}$ carry information about $\bD^\star$, but distance matrices are invariant to rotation and translation, and therefore it may only be possible to recover $\bX$ up to a rigid transformation.  Therefore, we make the following assumption throughout the paper.
\begin{assume}  To eliminate the translational ambiguity,  assume the points $\bx_1,\dots\bx_n\in \R^d$ are centered at the origin (i.e., $\sum_{i=1}^n \bx_i = {\bf 0}$).
\end{assume}
Define the {\em centering matrix} $\bV := \bI-\frac{1}{n}\mathbf{1}\mathbf{1}^T$. Note that under the above assumption, $\bX\bV = \bX$.  Note that $\bD^\star$ is determined by the Gram matrix $\bG^\star = \bX^T\bX$. In addition, $\bX$ can be determined from $\bG$ up to a unitary transformation.
 Furthermore, we will assume that the Gram matrix is ``centered'' so that $\bV\bG\bV = \bG$.  Centering is equivalent to assuming the underlying points are centered at the origin (e.g., note that $\bV\bG^\star \bV = \bV\bX^T\bX\bV = \bX^T\bX$).  It will be convenient in the paper to work with both the distance and Gram matrix representations, and the following identities will be useful to keep in mind. For any distance matrix $\bD$ and its centered Gram matrix $\bG$
\begin{eqnarray}
\bG & = & \bV\bD\bV \ , \label{D2G} \\
\bD & = &  \mbox{diag}(\bG)\b1^T-2\bG + \b1 \mbox{diag}(\bG)^T \ ,  \label{G2D}
\end{eqnarray}
where $\diag(\bG)$  is the column vector composed of the diagonal of $\bG$. In particular this establishes a bijection between centered Gram matrices and distance matrices. We refer the reader to \cite{Dattorro} for an insightful and thorough treatment of the properties of distance matrices.
We also define the set of all unique triplets 
$$\T \ := \ \big\{(i,j,k) \, : \, 1\leq i \neq j\neq k \leq n, j<k\big\} \ . $$
\begin{assume} The observed triplets in $\S$ are drawn independently and unifomly from $\T$. \end{assume}

\section{Prediction Error Bounds}
\label{predict}
For $t\in \T$ with $t=(i,j,k)$ we define $L_t$ to be the linear operator satisfying $L_t(\bX^T\bX) = \|x_i-x_j\|^2-\|x_i-x_k\|^2$ for all $t\in \T$.  In general, for any Gram matrix $\bG$
 $$L_t(\bG) \  := \ G_{jj}-2G_{ij}-G_{kk}+2G_{ik}. $$
We can naturally view $L_t$ as a linear operator on $\mS^n_{+}$, the space of $n\times n$ symmetric positive semidefinite matrices. We can also represent $\bL_t$ as a symmetric $n\times n$ matrix $\bL_t$ that is zero everywhere except on the submatrix corresponding to $i,j,k$ which has the form
\[\left[
\begin{array}{rrr}
0  & -1 & 1\\
-1 &  1 & 0\\
1  &  0 & -1
\end{array}
\right]\]
and then
\[L_{t}(\bG) \ := \ \langle \bL_t, \bG \rangle \ \]
where $\langle A, B \rangle = \text{vec}(A)^T \text{vec}(B)$ for any compatible matrices $A,B$. 
Ordering the elements of $\mathcal{T}$ lexicographically, we arrange all the $\bL_t(\bG)$ together to define the $n{n-1\choose{2}}$-dimensional vector 
\begin{align}\label{eqn:L-operator}
\L(\bG)  = [\bL_{123}(\bG),\bL_{124}(\bG),\cdots, \bL_{ijk}(\bG), \cdots ]^T.
\end{align}
Let $\ell(y_t \langle \bL_t, \bG\rangle)$ denote a loss function. For example we can consider the $0-1$ loss $\ell(y_t\langle \bL_t, \bG\rangle) = \1_{\{ \text{sign}\{ y_t \langle \bL_t, \bG \rangle\} \neq 1\} }$, the hinge-loss $\ell(y_t \langle \bL_t, \bG\rangle) = \max\{0,1-y_t\langle \bL_t, \bG\rangle\}$, or the logistic loss 
\begin{align}\label{eqn:logistic_loss}
\ell(y_t \langle \bL_t, \bG\rangle) = \log(1+\exp(-y_t\langle \bL_t, \bG\rangle)).
\end{align}
Let $p_t:=\P(y_t=-1)$ and take the expectation of the loss with respect to both the uniformly random selection of the triple $t$ and the observation $y_t$, we have the risk of $\bG$
\begin{eqnarray*}
R(\bG) &:= & \E[\ell(y_t \langle \bL_t, \bG \rangle)] = \frac{1}{|\T|}\sum_{t\in \T} p_t\ell( -\langle \bL_t, \bG  \rangle )+(1-p_t)\ell( \langle \bL_t, \bG\rangle ).
\end{eqnarray*}
Given a set of observations $\S$ under the model defined in the problem statement, the empirical risk is,
\begin{align}\label{eqn:empirical_risk}
\widehat R_\S(\bG) = \frac{1}{|\S|}\sum_{t\in \S} \ell(y_t \langle \bL_t, \bG\rangle)
\end{align}
which is an unbiased estimator of the true risk: $\E[\widehat R_\S(\bG)] = R(\bG)$.
For any $\bG \in \mS^n_{+}$, let $\|\bG\|_*$ denote the nuclear norm and $\|\bG\|_\infty := \max_{ij}|\bG_{ij}|$. Define the constraint set
\begin{eqnarray}\label{eqn:gram_matrix_domain}
\G_{\lambda,\gamma} \ := \ \{\bG \in \mS^n_{+} : \|\bG\|_* \leq \lambda, \|\bG\|_\infty \leq \gamma\} \ .
\end{eqnarray}
We estimate $\bG^\star$ by  solving the optimization
\begin{eqnarray}\label{eqn:nuclear-norm-optimization}
\min_{\bG \in \G_{\lambda,\gamma}} \widehat R_\S(\bG)  \ . 
\label{convex}
\end{eqnarray}
Since $\bG^\star$ is positive semidefinite, we expect the diagonal entries of $\bG^\star$ to bound the off-diagonal entries. So an infinity norm constraint on the diagonal guarantees that the points $\bx_1,\dots,\bx_n$ corresponding to $\bG^\star$ live inside a bounded $\ell_2$ ball. The $\ell_\infty$ constraint in \eqref{eqn:gram_matrix_domain} plays two roles: 1) if our loss function is Lipschitz, large magnitude values of $\langle \bL_t,\bG\rangle$ can lead to large deviations of $\widehat R_\S(\bG)$ from $R(\bG)$; bounding $||\bG||_\infty$ bounds $|\langle \bL_t,\bG\rangle|$. 2) Later we will define $\ell$ in terms of the link function $f$ and as the magnitude of $\langle \bL_t,\bG\rangle$ increases the magnitude of the derivative of the link function $f$ typically becomes very small, making it difficult to ``invert''; bounding $||\bG||_\infty$ tends to keep $\langle \bL_t,\bG\rangle$ within an invertible regime of $f$.

\begin{theorem}\label{thm:nuclear_generalization}\label{convexrecovery}
Fix $\lambda,\gamma$ and assume $\bG^\star \in \G_{\lambda,\gamma}$. Let $\widehat \bG$ be a solution to the program (\ref{eqn:nuclear-norm-optimization}). If the loss function $\ell( \cdot )$ is $L$-Lipschitz (or $|\sup_y \ell( y )| \leq L \max\{1,12\gamma\}$) then with probability at least $1-\delta$,    
\[
R(\widehat \bG)-R(\bG^\star) \leq \frac{4L\lambda}{|\S|}\left(\sqrt{\frac{18 |\S| \log(n)}{n}} + \frac{\sqrt{3}}{3}\log n\right)+ L \gamma \sqrt{\frac{288 \log{{2}/{\delta}}}{|\S|}}
\]
\end{theorem}

\begin{proof}
The proof follows from standard statistical learning theory techniques, see for instance \cite{boucheron2005theory}. By the bounded difference inequality, with probability $1-\delta$ 
\begin{align*}
R(\widehat \bG)-R(\bG^\star) &=  R(\widehat \bG) - \widehat{R}_\S(\widehat \bG) +\widehat{R}_\S(\widehat \bG) - \widehat{R}_\S(\bG^\star) + \widehat{R}_\S(\bG^\star) - R(\bG^\star) \\
&\leq 2\sup_{\bG\in \G_{\lambda,\gamma}}|\widehat R_\S(\bG)-R(\bG)| \leq  2\E[\sup_{\bG\in \G_{\lambda,\gamma}}|\widehat R_\S(\bG)-R(\bG)|]+\sqrt{\frac{2B^2\log{{2}/{\delta}}}{|S|}}
\end{align*}
where $\sup_{\bG \in \G_{\lambda,\gamma}} \ell( y_t \langle \bL_t , \bG \rangle ) -  \ell( y_{t'} \langle \bL_{t'} , \bG \rangle ) \leq \sup_{\bG \in \G_{\lambda,\gamma}} L | \langle y_t \bL_t - y_{t'} \bL_{t'}, \bG \rangle| \leq 12 L \gamma =: B$ using the facts that $\bL_t$ has $6$ non-zeros of magnitude $1$ and $||\bG||_\infty \leq \gamma$.

Using standard symmetrization and contraction lemmas, we can introduce Rademacher random variables $\epsilon_t \in \{-1,1\}$ for all $t\in \S$ so that
\[
\E\sup_{\bG\in \G_{\lambda,\gamma}}|\widehat R_\S(\bG)-R(\bG)| \leq \E\sup_{\bG\in \G_{\lambda,\gamma}}\frac{2L}{|\S|}\left |\sum_{t\in \S} \epsilon_{t}\langle \bL_{t}, \bG\rangle\right|.
\]
The right hand side is just the Rademacher complexity of  $\G_{\lambda,\gamma}.$ By definition,
\[
\{\bG : \|\bG\|_{\ast} \leq \lambda\} = \lambda\cdot\text{conv}(\{ uu^T:|u|=1\}). 
\]
where $\text{conv}(U)$ is the convex hull of a set $U$. Since the Rademacher complexity of a set is the same as the Rademacher complexity of it's closed convex hull, 
\begin{align*}
\E \sup_{\bG\in \G_{\lambda,\gamma}}\left |\sum_{t\in \S} \epsilon_t\langle \bL_t, \bG\rangle\right| &\leq \lambda \E \sup_{|u|=1}\left |\sum_{t\in \S} \epsilon_t\langle \bL_t, uu^T\rangle\right| = \lambda \E \sup_{|u|=1}\left | u^T\left( \sum_{t\in \S} \epsilon_t \bL_t \right )u \right | 
\end{align*}
which we recognize is just $\lambda \E\|\sum_{t\in \S} \epsilon_t \bL_t\|$.
By \cite[6.6.1]{tropp2015introduction} we can bound the operator norm $\|\sum_{t\in S} \epsilon_t \bL_t\|$ in terms of the variance of $\sum_{t\in\S} \bL_t^2$ and the maximal eigenvalue of $\max_t \bL_t.$ These are computed in Lemma 1 given in the supplemental materials. Combining these results gives, 
\[
 \frac{2L\lambda}{|\S|} \E\|\sum_{t\in \S} \epsilon_t\bL_t\| 
\leq \frac{2L\lambda}{|\S|}\left(\sqrt{\frac{18|\S|\log(n)}{n}} + \frac{\sqrt{3}}{3}\log n\right). 
\]
\end{proof}

We remark that if $\bG$ is a rank $d<n$ matrix then
\[
\|\bG\|_{\ast}\leq \sqrt{d}\|\bG\|_{F}\leq \sqrt{d}n\|\bG\|_{\infty}
\]
so if $\bG^{\star}$ is low rank, we really only need a bound on the infinity norm of our constraint set. Under the assumption that $\bG^\star$ is rank $d$ with $||\bG^\star||_\infty \leq \gamma$ and we set $\lambda = \sqrt{d}{n} \gamma$, then Theorem~\ref{convexrecovery} implies that for $|S| > n\log{n}/161$
\begin{align*}
R(\widehat \bG)-R(\bG^\star) \leq 8L\gamma \sqrt{\frac{18 d n \log(n)}{|\S|}}+ L \gamma \sqrt{\frac{288 \log{{2}/{\delta}}}{|\S|}}
\end{align*}
with probability at least $1-\delta$.
The above display says that $|\S|$ must scale like $dn \log(n)$ which is consistent with known finite sample bounds \cite{jamieson2011low}.

\section{Maximum Likelihood Estimation}
\label{mle_sec}
We now turn our attention to recovering metric information about $\bG^\star$. 
Let $\S$ be a collection of triplets sampled uniformly at random with replacement and let $f:\R\rightarrow (0,1)$ be a known probability function governing the observations.
Any link function $f$ induces a natural loss function $\ell_f$, namely, the negative log-likelihood of a solution $\bG$ given an observation $y_t$ defined as 
\begin{align*}
\ell_f( y_t \langle \bL_t, \bG \rangle ) = \1_{y_{t} = -1}\log(\tfrac{1}{f(\pg)}) + \1_{y_{t} = 1}\log(\tfrac{1}{1-f( \pg)})
\end{align*}
For example, the logistic link function of \eqref{eqn:logistic_link} induces the logistic loss of \eqref{eqn:logistic_loss}. Recalling that $\P( y_t = -1 ) = f( \pg )$
we have 
\begin{align*}
\E[ \ell_f( y_t \langle \bL_t, \bG \rangle ) ] &= f(\pt)\log(\tfrac{1}{f( \pg)})+(1-f(\pt)\log(\tfrac{1}{1-f(\pg)}) \\
&= H( f(\pt) ) + KL( f(\pt) | f(\pg) )
\end{align*}
where $H(p) = p \log(\tfrac{1}{p}) + (1-p) \log(\tfrac{1}{1-p})$ and $KL(p,q)=p\log(\tfrac{p}{q}) + (1-p)\log(\tfrac{1-p}{1-q})$ are the entropy and KL divergence of Bernoulli RVs with means $p,q$. Recall that $||\bG||_\infty \leq \gamma$ controls the magnitude of $\langle \bL_t, \bG \rangle$ so for the moment, assume this is small. Then by a Taylor series $f( \langle \bL_t, \bG \rangle ) \approx \frac{1}{2} + f'(0) \langle \bL_t, \bG \rangle$ using the fact that $f(0)=\frac{1}{2}$, and by another Taylor series we have
\begin{align*}
KL( f(\pt) | f(\pg) ) &\approx KL( \tfrac{1}{2} + f'(0) \langle \bL_t, \bG^\star \rangle | \tfrac{1}{2} + f'(0) \langle \bL_t, \bG \rangle ) \\
&\approx 2 f'(0)^2 ( \langle \bL_t, \bG^\star - \bG \rangle )^2.
\end{align*}
Thus, recalling the definition of $\L(\bG)$ from \eqref{eqn:L-operator} we conclude that if $\widetilde{\bG} \in \arg\min_{\bG} R(\bG)$ with $R(\bG) =  \frac{1}{|\T|}\sum_{t\in \T} \E[ \ell_f( y_t \langle \bL_t, \bG \rangle )]$ then one would expect $\mathcal{L}(\widetilde{\bG}) \approx \mathcal{L}(\bG^\star)$. Moreover, since $\widehat{R}_\S(\bG)$ is an unbiased estimator of $R(\bG)$, one expects $\L(\widehat{\bG})$ to approximate $\L(\bG^\star)$. The next theorem, combined with Theorem~\ref{thm:nuclear_generalization}, formalizes this observation; its proof is found in the appendix.
\begin{theorem}\label{thm:linear-operator-recovery}
Let $C_f = \min_{t \in \T} \inf_{\bG\in \G_{\lambda, \gamma}} | f'\big( \langle \bL_t, \bG \rangle\big)|$ where $f'$ denotes the derivative of $f$. Then for any $\bG$
\begin{align*}
\frac{2C_f^2}{|\T|} \|\L(\bG)-\L(\bG^\star)\|_F^2 \ \leq \ R(\bG)-R(\bG^\star) \ .
\end{align*}
\end{theorem}
Note that if $f$ is the logistic link function of \eqref{eqn:logistic_link} then its straightforward to show that $| f'\big( \langle \bL_t, \bG \rangle\big)| \geq \tfrac{1}{4}\exp(-| \langle \bL_t, \bG \rangle|) \geq \tfrac{1}{4} \exp( - 6 ||\bG||_\infty )$ for any $t$, $\bG$ so it suffices to take $C_f = \tfrac{1}{4} \exp( -6 \gamma )$.

\section{Maximum Likelihood Embedding}
\label{embedding}
\setcounter{lemma}{1}
In this section, let $\widehat \bG$ be the maximum likelihood estimator; i.e., a solution to the optimization \label{eqn:nuclear-norm-optimization} with $L$-Lipschitz log-likelihood loss function $\ell_f$ for a fixed $\lambda,\gamma$. 
We have shown that the maximum likelihood estimator allows us to bound $\|\L(\widehat \bG)-\L(\bG^\star)\|_F^2$.  In this section, we discuss how this bound can lead us to recover an embedding. For the analysis in this section, it will be convenient to work with distance matrices in addition to Gram matrices.
Analogous to the operators $\bL_t(\bG)$ defined above, we define the operators $\Delta_t$ for $t\in \T$ satisfying,
$$\Delta_t(\bD) := D_{ij}-D_{ik} \equiv L_t(\bG) \ . $$  
We will view the $\Delta_t$ as linear operators on the space of symmetric hollow $n\times n$ matrices $\mS^n_{h}$, which includes distance matrices as special cases. As with $\L$,  we can arrange all the $\Delta_t$ together, ordering the $t\in \T$ lexicographically, to define the $n{n-1\choose{2}}$-dimensional vector 
\[
\Delta(\bD)  = [D_{12}-D_{13},\cdots, D_{ij} - D_{ik}, \cdots ]^T.
\]
We will use the fact that $\L(\bG) \equiv \Delta(\bD)$ heavily. Because $\Delta(\bD)$ consists of differences of matrix entries, $\Delta$ has a non-trivial kernel.  However, it is easy to see that $\bD$ can be recovered given $\Delta(\bD)$ {\em and} any one off-diagonal element of $\bD$, so the kernel is $1$-dimensional.  Also, the kernel is easy to identify by example.  Consider the regular simplex in $d$ dimensions. The distances between all $n=d+1$ vertices are equal and the distance matrix can easily be seen to be $\b1\b1^T - I.$ Thus $\Delta(\bD) = {\bf 0}$ in this case. This gives us the following simple result.
\begin{lemma}
Let $\mS^n_{h}$ denote the space of symmetric hollow matrices, which includes all distance matrices.
For any $\bD\in \mS^n_{h}$, the set of linear functionals $\{\Delta_t(\bD)$, $t\in \T\}$ spans an ${n\choose{2}}-1$ dimensional subspace of $\mS^n_{h}$, and the $1$-dimensional kernel is given by the span of $\b1\b1^T-\bI$.
\label{lemma_L} \end{lemma}
So we see that the operator $\Delta$ is not invertible on $\mS^n_{h}$.  Define $\bJ:= \b1\b1^T - \bI$. For any $\bD$, let $\bC$, \textit{the centered distance matrix}, be the component of $\bD$ orthogonal to the kernel of $\L$ (i.e., $\mbox{tr}(\bC\bJ)=0$).  Then we have the orthogonal decomposition
$$\bD \ = \ \bC \ + \ \sigma_D \, \bJ \ , $$
where $\sigma_D = \mbox{trace}(\bD\bJ)/\|\bJ\|_F^2$.  Since $\bG$ is assumed to be centered, the value of $\sigma_D$ has a simple interpretation:
\begin{eqnarray}\label{eqn:sigmad} 
\sigma_D = \frac{1}{2{n\choose{2}}} \sum_{1\leq i\leq j\leq n} D_{ij} = \frac{2}{n-1}\sum_{1\leq i\leq n} \langle x_i,x_i \rangle = \frac{2\|\bG\|_{\ast}}{n-1},
\end{eqnarray}
the average of the squared distances or alternatively a scaled version of the nuclear norm of $\bG$. Let $\widehat \bD$ and $\widehat\bC$ be the corresponding distance and centered distance matrices.

Now following the notation of sections 2 and 3, assume that there is a true Gram matrix $\bG^\star$ and a link function $f$ as in section 3, and assume that we have observed a set of triples $\S.$  
Though $\Delta$ is not invertible on all $\mS^n_{h}$, it is invertible on the subspace orthogonal to the kernel, namely $\bJ^\perp.$ So if $\L(\widehat \bG) = \Delta(\widehat \bD)$ is close to $\L(\bG^\star)$ we expect $\Delta(\widehat \bC)$ to be close to $\Delta(\bC^\star)$. The next theorem quantifies this.

\begin{theorem}\label{thm:Crecovery}
Consider the setting of Theorems~\ref{thm:nuclear_generalization} and \ref{thm:linear-operator-recovery} and let $\widehat{\bC},\bC^\star$ be defined as above. Then
$$\frac{1}{2{n\choose 2}}\|\widehat{\bC}-\bC^\star\|_F^2 \ \leq \frac{L\lambda}{4C_f^2 |\S|}\left(\sqrt{\frac{18 |\S| \log(n)}{n}} + \frac{\sqrt{3}}{3}\log n\right)+ \frac{L \gamma}{4C_f^2} \sqrt{\frac{288 \log{{2}/{\delta}}}{|\S|}}$$
\end{theorem}
\begin{proof}
By combining Theorem~\ref{thm:linear-operator-recovery} with the prediction error bounds obtainined in \ref{thm:nuclear_generalization} we see that 
\[
\frac{2C_f^2}{n\binom{n-1}{2}}\|L(\widehat \bG)-L(\bG^\star)\|^2_F \leq \frac{4L\lambda}{|\S|}\left(\sqrt{\frac{18 |\S| \log(n)}{n}} + \frac{\sqrt{3}}{3}\log n\right)+ L \gamma \sqrt{\frac{288 \log{{2}/{\delta}}}{|\S|}}.
\]
Next we employ the following \textit{restricted isometry property} of $\Delta$ on the subspace $\bJ^\perp$ whose proof is in the supplementary materials. 
\begin{lemma}\label{lem:rip}
Let $\bD$ and $\bD'$ be two different distance matrices of $n$ points in $\R^d$ and $\R^{d'}$.  Let $\bC$ and $\bC'$ be the components of $\bD$ and $\bD'$ orthogonal to $J$. Then
$$n \|\bC-\bC'\|_F^2 \ \leq \ \|\Delta(\bC)-\Delta(\bC')\|^2 = \|\Delta(\bD)-\Delta(\bD')\|^2\ \leq \ 2(n-1)\|\bC-\bC'\|_F^2 \ .$$
\end{lemma}
The result then follows.
\end{proof}
This implies that by collecting enough samples, we can recover the centered distance matrix. By applying the discussion following Theorem~\ref{thm:nuclear_generalization} when $\bG^\star$ is rank $d$, we can state an upperbound of $\frac{1}{2{n\choose 2}}\|\widehat{\bC}-\bC^\star\|_F^2 \ \leq O\left( \frac{L\gamma}{C_f^2} \sqrt{\frac{dn\log(n)+\log(1/\delta)}{|\S|}} \right)$.
However, it is still not clear that this is enough to recover $\bD^\star$ or $\bG^\star$. 
 
Remarkably, despite this unknown component being in the kernel, we show next that it can be recovered.

Though $\Delta$ is not invertible on $\mS^n_h$ in general, we can provide a heuristic argument for why we might still expect it to be invertible on the (non-linear) space of low rank distance matrices. Since a distance matrix of $n$ points in $\mathbb{R}^d$ is at most rank $d+2$, the space of all distance matrices has at most $n(d+2)$ degrees of freedom. Now $\Delta$ is a rank $\binom{n}{2}-1$ operator and since $\binom{n}{2}-1\gg n(d+2)$ for $n>d+2$, it is not unreasonable to hope that the entries of $\Delta(\bD)$ provide enough information to parametrize the set of rank $d+2$ Euclidean distance matrices.
In fact as the next theorem will show, the intuition that $\Delta$ is invertible on the set of low rank distance matrices is in fact true. Perhaps even more surprisingly, merely knowing the centered distance matrix $\bC$ uniquely determines $\bD.$  

\begin{theorem}
Let $\bD$ be a distance matrix of $n$ points in $\R^d$, let $\bC$ be the component of $\bD$ orthogonal to the kernel of $\L$, and let $\lambda_2(\bC)$ denote the second largest eigenvalue of $\bC$. If $n>d+2$, then
$$\bD \ = \ \bC \ + \ \lambda_2(\bC) \, \bJ \ . $$
\label{drep}
\end{theorem}
This shows that $\bD$ is uniquely determined as a function of $\bC$.  Therefore, since $\Delta(\bD) = \Delta(\bC)$ and because $\bC$ is orthogonal to the kernel of $\Delta$, the distance matrix $\bD$ can be recovered from $\Delta(\bD)$, even though the linear operator $\Delta$ is non-invertible.

We now provide a proof of Theorem 4. In preparing this paper for publication we became aware of an alterntive proof of this theorem \cite{TarazagaGallardo}. We nevertheless present our independently derived proof next for completeness. 

\begin{proof}
To prove Theorem~\ref{drep} we first state two simple lemmas, which are proved in the supplementary material.
\begin{lemma}
\label{negsemdef}
Let $\bD$ be a Euclidean distance matrix on $n$ points. Then $\bD$ is negative semidefinite on the subspace 
\[
\b1^{\perp} := \{\bx\in \R^n| \b1^T\bx = 0\}.
\]
Furthermore, $\ker(\bD)\subset \b1^{\perp} $. 
\end{lemma}

\begin{lemma}
\label{Deigenvalues}
If $\bD$ is an $n\times n$ distance matrix of rank $r$, then
$\bD$ has a single positive eigenvalue, $n-r$ eigenvalues equal to $0$, and $r-1$ negative eigenvalues.
\end{lemma}

We will use the following notation.  For any matrix $\bM$, let $\lambda_i(\bM)$ denote its $i$th largest eigenvalue. We will show that for $\sigma>0$ and any $n\times n$ distance matrix $\bD$ with $n > d+2$,
\[
\lambda_2(\bD - \sigma \bJ) = \sigma \ .
\]
Since $\bC = \bD - \sigma_D \bJ$, this proves the theorem.

Note that $\lambda_i(\bD-\sigma\b1\b1^T+\sigma\bI) = \lambda_i(\bD-\sigma\b1\b1^T)+\sigma$, for $1\leq i\leq n$ and $\sigma$ arbitrary. So it suffices to show that $\lambda_2(\bD-\sigma\b1\b1^T) = 0$.

Let the rank of $\bD$ be $r$ where $r\leq d+2$ and consider an eigendecomposition $\bD = \bU\bS\bU^T$ where $\bU$ is unitary. If $\bv$ is an eigenvector of $\bD-\sigma\b1\b1^T$ with eigenvalue $\lambda$, then
\begin{align*}
\bD\bv - \sigma\b1\b1^T\bv &= \lambda \bv\\
\Rightarrow \bU\bS\bU^T\bv - \sigma\b1(\bU^T\b1)^T\bU^T\bv &= \lambda \bv\\
\Rightarrow \bS\bU^T\bv - \sigma\bU^T\b1(\bU^T\b1)^T\bU^T\bv &= \lambda \bU^T\bv. 
\end{align*}
So $\bU^T\bv$ is an eigenvalue of $\bS-\sigma\bU^T\b1(\bU^T\b1)^T$ with the same eigenvalue $\lambda.$ Denoting $\bz=\bU^T\b1$, we see that   $\lambda_i(\bD-\sigma\b1\b1^T) = \lambda_i(\bS-\sigma\bz\bz^T)$ for all $1\leq i\leq n$.

By a result on interlaced eigenvalues \cite[Cor 4.3.9]{HornJohnson},
\[\lambda_1(\bS) \geq \lambda_1(\bS - \sigma\bz\bz^T)\geq \lambda_2(\bS)\geq \lambda_2(\bS - 
\sigma\bz\bz^T)\geq \cdots \geq \lambda_n(\bS)\geq \lambda_n(\bS-\sigma\bz\bz^T).\]

From the above inequality and using the fact that $\lambda_2(\bS) = 0$ (since $n>d+2$), it is clear that $\bS-\sigma\bz\bz^T$ has at least one positive eigenvalue. It is also clear that since $\sigma>0$, each eigenvalue of $\bS-\sigma\bz\bz^T$ is bounded above by an eigenvalue of $\bS$, so $\bS-\sigma \bz\bz^t$ has at least $r-1$ negative eigenvalues. Hence $\dim \ker \bS-\sigma\bz\bz^T \leq n-r.$ 

If $\bx\in \ker \bD$, then $\b1^T\bx = 0$ by lemma \ref{negsemdef}. Thus $\bx\in \ker \bD-\sigma\b1\b1^T$ and so $\bU^T\bx\in \ker \bS - \sigma\bz\bz^T.$ This implies that $\bU^T\ker \bD\subset \ker \bS - \sigma\bz\bz^T$ and $\dim \ker \bS-\sigma\bz\bz^T \geq n-r.$ However from the above, $\dim \ker \bS-\sigma\bz\bz^T \leq n-r.$ Hence we can conclude that $\lambda_2(\bS-\sigma\bz\bz^T) = 0.$ 
\end{proof}

The previous theorem along with Theorem~\ref{thm:Crecovery} guarantees that we can recover $\bG^\star$ as we increase the number of triplets sampled.  We summarize this in our final theorem, which follows directly from Theorems~\ref{thm:Crecovery} and \ref{drep}.

\begin{theorem}\label{thm:Grecovery}
Assume $n>d+2$ and consider the setting of Theorems~\ref{thm:nuclear_generalization} and \ref{thm:linear-operator-recovery}. As $|\S|\rightarrow \infty$, the maximum likelihood estimator 
\[
\widehat \bG = -\frac{1}{2}\bV(\widehat \bC+\lambda_2(\widehat \bC)\bJ)\bV
\]
converges to $\bG^\star$.
\end{theorem}

\section{Experimental Study}
\label{exp}
The section empirically studies the properties of our estimators suggested by our theory. 
It is {\em not} an attempt to perform an exhaustive empirical evaluation of different embedding techniques; for that see \cite{jamieson2015next,tamuz2011adaptively,van2012stochastic,agarwal2007generalized}.  
In what follows each of the $n$ points is generated randomly: $\bx_i \sim \mathcal{N}(0,\tfrac{1}{2d} I_d) \in \R^d$, $i=1,\dots,n$,  motivated by the observation that $$\E[ |\langle \bL_t, \bG^\star \rangle| ] \ = \ \E\big[ \big| \, \|\bx_{i}-\bx_j\|_2^2 - ||\bx_i-\bx_k||_2^2\, \big| \big] \ \leq \ \E\big[ \|\bx_i-\ \bx_j\|_2^2 \big] \ = \ 2 \E\big[ \|\bx_i\|_2^2 \big] \ = \ 1$$ for any triplet $t=(i,j,k)$. We perform $36$ trials to see the effect of different random samples of triplets.

We report the prediction error on a holdout set of $10,000$ triplets and the error in Frobenius norm of the estimated Gram matrix. 

\begin{figure}
  \begin{center}
    \includegraphics[width=.45\textwidth]{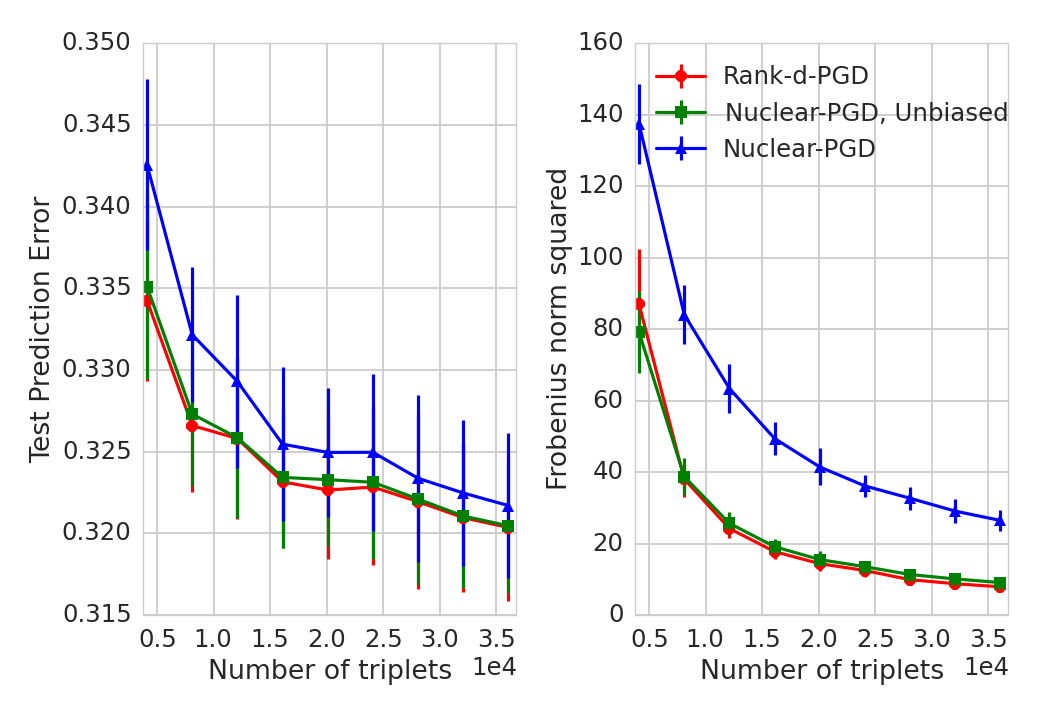} 
    \includegraphics[width=.45\textwidth]{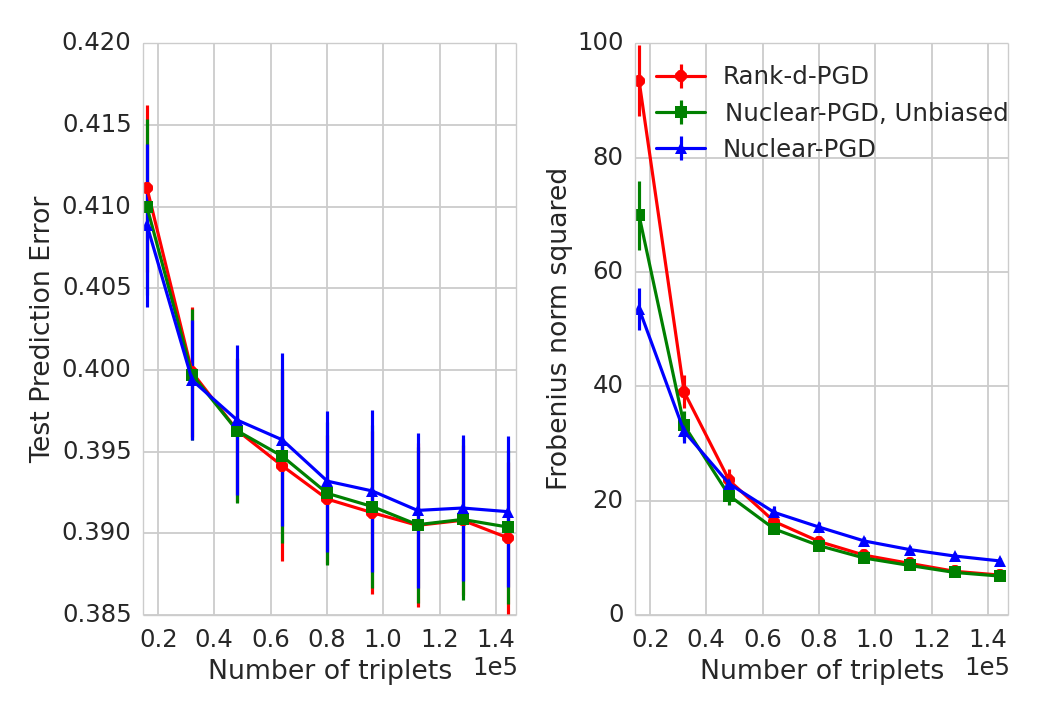}
  \end{center}
  \caption{$\bG^\star$ generated with $n=64$ points in $d=2$ and $d=8$ dimensions on the left and right.  \label{fig:varying_dim_small}}
\end{figure}

Three algorithms are considered. For each, the domain of the objective variable $\bG$ is the space of symmetric positive semi-definite matrices. None of the methods impose the constraint $\max_{ij}|G_{ij}|\leq \gamma$ (as done in our theoretical analysis), since this was used to simplify the analysis and does not have a significant impact in practice.  {\em Rank-d Projected Gradient Descent (PGD)} performs gradient descent on the objective $\widehat{R}_\S(\bG) = \frac{1}{|S|}\sum_{t \in S} \log( 1 + \exp(- y_t\langle \bL_t, \bG \rangle ) )$ with line search, projecting onto the subspace spanned by the top $d$ eigenvalues at each step (i.e. setting the bottom $n-d$ eigenvalues to $0$). {\em Nuclear Norm PGD} performs gradient descent on $\widehat{R}_\S (\bG)$ projecting onto the nuclear norm ball with radius $\|\bG^\star\|_*$, where $\bG^\star=\bX^T\bX$ is the Gram matrix of latent embedding. The nuclear norm projection can have the undesirable effect of shrinking the non-zero eigenvalues toward the origin.  To compensate for this potential bias, we also employ {\em Nuclear Norm PGD Debiased}, which takes the (biased) output of {\em Nuclear Norm PGD}, decomposes it into $\bU \bE \bU^T$ where $\bU \in \R^{n \times d}$
 are the top $d$ eigenvectors, and outputs $\bU \text{diag}(\widehat{\bs}) \bU^T$ where $\widehat{s} = \arg\min_{\bs \in \R^d} \widehat R_\S(\bU \text{diag}(\bs) \bU^T)$. This last algorithm is motivated by the fact observation that heuristics like minimizing  $\|\cdot\|_1$ or $\| \cdot \|_*$ are good at identifying the true support or basis of a signal, but output biased magnitudes \cite{rao2013conditional}.  {\em Rank-d PGD} and {\em Nuclear Norm PGD Debiased} are novel ordinal embedding algorithms.

Figure~\ref{fig:varying_dim_small} presents how the algorithms behave in $n=64$ and $d=2,8$ on the left and right, respectively. We observe that the unbiased nuclear norm solution behaves near-identically to the rank-$d$ solution in this case and remark that this was observed in all of our experiments (see the supplementary materials for a variety values of $n$, $d$, and scalings of $\bG^\star$). 
A popular technique for recovering rank $d$ embeddings is to perform (stochastic) gradient descent on $\widehat R_\S(\bU^T\bU)$ with objective variable $\bU \in \R^{n\times d}$ taken as the embedding \cite{jamieson2015next,tamuz2011adaptively,van2012stochastic}. 
In all of our experiments this method produced Gram matrices that were nearly identical to those produced by our Rank-$d$-PGD method, but Rank-$d$-PGD was an order of magnitude faster in our implementation. 
Also, in light of our isometry theorem, we can show that the Hessian of $\E[\widehat R_\S(\bG)]$ is nearly a scaled identity, which leads us to hypothesize that a globally optimal linear convergence result for this non-convex optimization may be possible using the techniques of \cite{oymak2015sharp,shen2016tight}.
Finally, we note that previous literature has reported that nuclear norm optimizations like {\em Nuclear Norm PGD} tend to produce less accurate embeddings than those of non-convex methods \cite{tamuz2011adaptively,van2012stochastic}.
We see from the results that the {\em Nuclear Norm PGD Debiased} appears to close the performance gap between the convex and non-convex solutions.

\section{Future work}
For any fixed set of $n$ points in $\R^d$ and randomly selected distance comparison queries, our results show that the component of $\bG^\star$ (the Gram matrix associated with the points) in the span of all possible distance comparisons can be accurately recovered from $O(dn\log n)$ queries, and we conjecture these bounds to be tight. Moreover, we proved the existence of an estimator $\widehat{\bG}$ such that as the number of queries grows, we have $\widehat{\bG} \rightarrow \bG^\star$. A focus of our ongoing work is characterizing finite sample bounds for the rate at which $\widehat{\bG} \rightarrow \bG^\star$. One way of approaching such a result is showing
\[ c_1\|\bD-\bD'\|_F^2 \ \leq \ \|\Delta(\bD)-\Delta(\bD')\|^2\ \leq \ c_2\|\bD-\bD'\|_F^2 \]
for the tightest possible constants $c_1$ and $c_2$.
By inspecting our proofs, we can provide satisfying values for $c_1$ and $c_2$, however, they differ by a factor of $n$ and hence we do not believe these to be tight. Empirically we observe that if $n\gg d+2$, then the ratio of $c_1$ to $c_2$ appears to be independent of $n$.

\newpage
\bibliographystyle{unsrt}
\bibliography{nips16_bib}
\newpage
\section{Supplementary Materials for ``Finite Sample Error Bounds for Ordinal Embedding''}
\subsection{Proof of Lemma 1}
\setcounter{lemma}{0}
\begin{lemma}  For all $t\in T$,
\[\lambda_{1}(\bL_t) = \|\bL_{t}\| =\sqrt{3} \]
in addition if $n\geq 3$
\[\|\E_{t}[\bL_{t}^2]\| = \frac{6}{n-1}\leq \frac{9}{n}\]
\end{lemma}
\begin{proof}
Note that $\bL_{t}^3-3\bL_t = 0$ for all $t\in \T.$ Thus by the Cayley-Hamilton theorem, $\sqrt{3}$ is the largest eigenvalue of $\bL_t.$ A computation shows that the submatrix of $\bL_t^2$ corresponding to $i,j,k$ is 
\[
\begin{pmatrix}
2 & -1 & -1\\
-1 & 2 & -1\\
-1 & -1& 2\\
\end{pmatrix}
\]
and every other element of $\bL_t^2$ is zero. Summing over the $t\in \T$ then gives,
\[\E[\bL_t^2] = \frac{1}{n\binom{n-1}{2}}\sum_{t\in \T} \bL_t^2= 
\begin{pmatrix}
\frac{6}{n} & \frac{-6}{n(n-1)} & \cdots & \frac{-6}{n(n-1)}\\ 
\frac{-6}{n(n-1)} & \frac{6}{n} & \hdots & \frac{-6}{n(n-1)}\\ 
\vdots & \hdots & \hdots & \vdots\\ 
\frac{-6}{n(n-1)} & \hdots & \frac{-6}{n(n-1)} &  \frac{6}{n}
\end{pmatrix}
\]

This matrix can be rewritten as $\frac{6}{n}I-\frac{6}{n(n-1)}\bJ.$ The eigenvalues of $\bJ$ are $-1$ with multiplicity $n-1$ and $n-1$ with multiplicity $1.$ Hence the largest eigenvalue of $\E[\bL_t^2]$ is $\frac{6}{n-1}.$
\end{proof}

\subsection{Proof of Theorem \ref{thm:linear-operator-recovery}}
\begin{proof}
For $y,z\in (0,1)$ let $g(z) =  z\log{\frac{z}{y}}+(1-z)\log{\frac{1-z}{1-y}}$. Then $g'(z) = \log{\frac{z}{1-z}}-\log{\frac{y}{1-y}}$ and $g''(z) = \frac{1}{z(1-z)}.$ By taking a Taylor series around $y$,
$$g(z) \geq \frac{(z-y)^2/2}{\sup_{x\in [0,1]} x(1-x)}\geq 2(z-y)^2.$$ 

Now applying this to $z = f(\pt)$ and $y = f(\pg)$ gives
\begin{eqnarray*}
f(\pt)\log{\tfrac{f(\pt)}{f(\pg)}}+(1-f(\pt))\log\tfrac{1-f(\pt)}{1-f(\pg)} &\geq 2(f(\pt)-f(\pg))^2\\
&\geq 2C_f^2 (\pt-\pg)^2
\end{eqnarray*} 
where the last line comes from applying Taylor's theorem to $f$, $f(x)-f(y) \geq \inf_{z\in [x,y]} f'(z) (x-y)$ for any $x,y$.
Thus
\begin{eqnarray*}
R(\bG)-R(\bG^\star) & = &\frac{1}{|\T|}\sum_{t\in \T} f(\pt)\log{\tfrac{f(\pt)}{f(\pg)}}+(1-f(\pt))\log\tfrac{1-f(\pt)}{1-f(\pg)} \\
&\geq &\frac{2C_f^2}{|\T|}\sum_{t\in T} (\pt-\pg)^2\\
&= &\frac{2C_f^2}{|\T|} \sum_{t\in T} (\langle \bL_t,  \bG - \bG^\star \rangle)^2= \frac{2C_f^2}{|\T|} \| \L( \bG) - \L(\bG^{\ast}) \|^2_2.
\end{eqnarray*}
\end{proof}

\subsection{Proof of Lemma 3}
\setcounter{lemma}{2}
\begin{lemma}
Let $\bD$ and $\bD'$ be two different distance matrices of $n$ points in $\R^d$ and $\R^{d'}$ respectively.  Let $\bC$ and $\bC'$ be the components of $\bD$ and $\bD'$ orthogonal to $J$. Then
$$n \|\bC-\bC'\|_F^2 \ \leq \ \|\Delta(\bC)-\Delta(\bC')\|^2 = \|\Delta(\bD)-\Delta(\bD')\|^2\ \leq \ 2(n-1)\|\bC-\bC'\|_F^2 \ .$$
\end{lemma}
We can view the operator $\Delta$ defined above as acting on the space $\R^{\binom{n}{2}}$ where each symmetric hollow matrix is identified with vectorization of it's upper triangular component. With respect to this basis $\Delta$ is an $n\binom{n-1}{2}\times \binom{n}{2}$ matrix, which we will denote by $\Delta$.
Since $\bC$ and $\bC'$ are orthogonal to the kernel of $\Delta,$ the lemma follows immediately from the following characterization of the eigenvalues of $\Delta^T\Delta.$
\setcounter{lemma}{5}

\begin{lemma}
$\Delta^T\Delta:\mS^n_h\rightarrow \mS^n_h$ has the following eigenvalues and eigenspaces,
\begin{itemize}
\item Eigenvalue $0$, with a one dimensional eigenspace.
\item Eigenvalue $n$, with a $n-1$ dimensional eigenspace.
\item Eigenvalue $2(n-1)$, with a $\binom{n}{2}-n$ dimensional eigenspace.
\end{itemize}
\end{lemma}
\begin{proof}
The rows of $\Delta$ are indexed by triplets $t\in \T$ and columns indexed by pairs $i,j$ with $1\leq i<j\leq n$ and vice-versa for $\Delta^T.$ The row of $\Delta^T$ corresponding to the pair $i,j$ is supported on columns corresponding to triplets $t=(l,m,n)$ where $m<n$ and $l$ and one of $m$ or $n$ form the pair $i,j$ or $j,i$. Specifically, letting $[\Delta^T]_{(i,j), t}$ denote the entry of $\Delta^T$ corresponding to row $i,j$ and column $t$,
\begin{itemize}
\item if $l=i, m=j$ then $[\Delta^T]_{(i,j), t} = 1$
\item if $l=i, n=j$ then $[\Delta^T]_{(i,j), t} = -1$
\item if $l=j, m=i$ then $[\Delta^T]_{(i,j), t} = 1$
\item if $l=j, n=i$ then $[\Delta^T]_{(i,j), t} = -1$
\end{itemize} 
Using this one can easily check that
\begin{align}
[\Delta^T\Delta\bD]_{i,j} &= \sum_{(i,j,k)\in \T}D_{ij}-D_{ik}-\sum_{(i,k,j)\in \T}D_{ik}-D_{ij}+\sum_{(j,i,k)\in \T}D_{ji}-D_{jk}-\sum_{(j,k,i)\in \T} D_{jk}-D_{ji}\nonumber\\
&= 2(n-1)D_{ij} - \sum_{n\neq i} D_{in}-\sum_{n\neq j} D_{jn}. \label{deqn}
\end{align}

This representation allows us to find the eigenspaces mentioned above very quickly.

{\bf Eigenvalue $0.$} From the above discussion, we know the kernel is generated by $\bJ=\b1\b1^T-I$.

{\bf Eigenvalue $2(n-1).$} This eigenspace corresponds to all symmetric hollow matrices such that $\bD\b1 = {\bf 0}$.
For such a matrix each row and column sum is zero and so in particular, the sums in (\ref{deqn}) are both zero. Hence for such a $\bD$,
\[
[\Delta^T\Delta \bD]_{i,j} = 2(n-1)D_{ij} 
\]
The dimension of this subspace is $\binom{n}{2} - n$, indeed there are $\binom{n}{2}$ degree of freedom to choose the elements of $\bD$ and $\bD\b1 = {\bf 0}$ adds $n$ constraints. 

{\bf Eigenvalue $n.$} This eigenspace corresponds to the span of the matrices $\bD^{(i)}$ defined as,
\[
\bD^{(i)} = -n(\be_i\b1^T+\b1\be_i^T-2\be_i\be_i^T)+2\bJ
\]
where $\be_i$ is the standard basis vector with a $1$ in the $i$th row and $0$ elsewhere.
As an example,
\[\bD^{(1)} = 
\begin{pmatrix}
0&-n+2& \cdots& -n+2 \\
\vdots& 2&\cdots& 2\\
-n+2&2&\cdots& 0
\end{pmatrix}.
\]
If $i,j\neq m$, then $D_{ij}^{(m)}:=[\bD^{(m)}]_{ij} = 2$, and we can compute the row and column sums 
\[
\sum_{n\neq i} D_{in}^{(m)} = \sum_{n\neq j} D_{jn}^{(m)} = 2(n-2)-n+2 = n-2.
\]
This implies that $D_{ij}^{(m)} = n$, and so by (\ref{deqn})
\begin{align*}
[\Delta^T\Delta \bD^{(m)}]_{i,j} &= 2(n-1)\cdot 2-(n-2)-(n-2)\ = \ 2n \ = \ nD_{ij}^{(m)} \ .
\end{align*}
Otherwise, without loss of generality we can assume that $i = m, j\neq m$ in which case, $[\bD^{(m)}]_{ij} = -n+2$, the row and columns sums can be computed as
\[
\sum_{n\neq i} D_{in}^{(m)} = (n-1)(-n+2)
\]
and
\[
\sum_{n\neq j} D_{in}^{(m)} = n-2.
\]
Putting it all together,
\begin{align*}
[\Delta^T\Delta \bD^{(m)}]_{m,j} &= 2(n-1)\cdot(-n+2)-(n-1)(-n+2)-(n-2)\\
&= (n-1)(-n+2) +(-n+2)\\
&= n(-n+2)\\
&= nD_{m,j}^{(m)}
\end{align*}
and $\Delta^T\Delta \bD = n \bD.$ Note that the dimension of $\text{span}{\langle \bD^{(i)}\rangle } = n-1$ since
\[
\sum_{m} \bD^{(m)} = 0.
\]
\end{proof}

\subsection{Proofs of Lemmas 4 and 5}
\setcounter{lemma}{3}
\begin{lemma}
Let $\bD$ be a Euclidean distance matrix on $n$ points. Then $\bD$ is negative semidefinite on the subspace 
\[
\b1^{\perp} := \{\bx\in \R^n| \b1^T\bx = 0\}.
\]
Furthermore, $\ker(\bD)\subset \b1^{\perp} $. 
\end{lemma}

\begin{proof}
The associated Gram matrix $\bG =  -\frac{1}{2}\bV\bD\bV$ is a positive semidefinite matrix. For $\bx\in \b1^{\perp}$, $Jx = -x$  so 
\[
x^T\left(-\frac{1}{2}\bV\bD\bV\right)x = -\frac{1}{2} \bx^T\bD\bx \leq 0  
\]
establishing the first part of the theorem. Now if $\bx\in \ker \bD$,
\begin{align*}
0 \leq -\frac{1}{2} \bx^T\bV\bD\bV\bx = -\frac{1}{2}\bx^T\b1\b1^T\bD\b1\b1^T\bx = -\frac{1}{2}\b1^T\bD\b1 (\b1^T\bx)^2\leq 0 \ , 
\end{align*}
where the last inequality follows from the fact that $\b1^T\bD\b1 > 0$ since $\bD$ is non-negative. Hence $\b1^T \bx = 0$ and $\ker \bD\subset \b1^{\perp}.$
\end{proof}

Let $\lambda_i(\bM)$ denote the $i$-th largest eigenvalue of a matrix $\bM.$

\begin{lemma}
If $\bD$ is an $n\times n$ distance matrix of rank $r$, then
\begin{itemize}
\item $\bD$ has a single positive eigenvalue
\item $\bD$ has $n-r$ zero eigenvalues. 
\item $\bD$ has $r-1$ negative eigenvalues.
\end{itemize}
\end{lemma}
\begin{proof}
From above we know that $r\leq d+2.$ We use the Courant-Fisher theorem \cite{HornJohnson} .
\begin{align*}
\lambda_1(\bD) &= \max_{\dim S = 1} \min_{0\neq \bx\in S} \frac{\bx^T \bD \bx}{\bx^T\bx}\\
&\geq \frac{\b1^T \bD\b1}{\b1^T\b1} > 0. 
\end{align*}

So we see the largest eigenvalue of $\bD$ is necessarily positive. Now for $2\leq k\leq n$, 
\begin{align*}
\lambda_k(\bD) &= \min_{\dim S = k} \max_{0\neq \bx\in S} \frac{\bx^T \bD \bx}{\bx^T\bx}\\
&\leq \min_{U\subset 1^{\perp},\dim U = k} \max_{0\neq \bx\in U} \frac{\bx^T\bD\bx}{\bx^T\bx}\\ 
& \leq 0,
\end{align*}
where the last step follows from the negative definiteness of $\bD$ on $\b1^{\perp}$.
\end{proof}

\section{Additional Empirical Results}

\begin{figure}
\centering
\begin{minipage}{.45\textwidth}
\begin{center}
    \includegraphics[width=\imwidth]{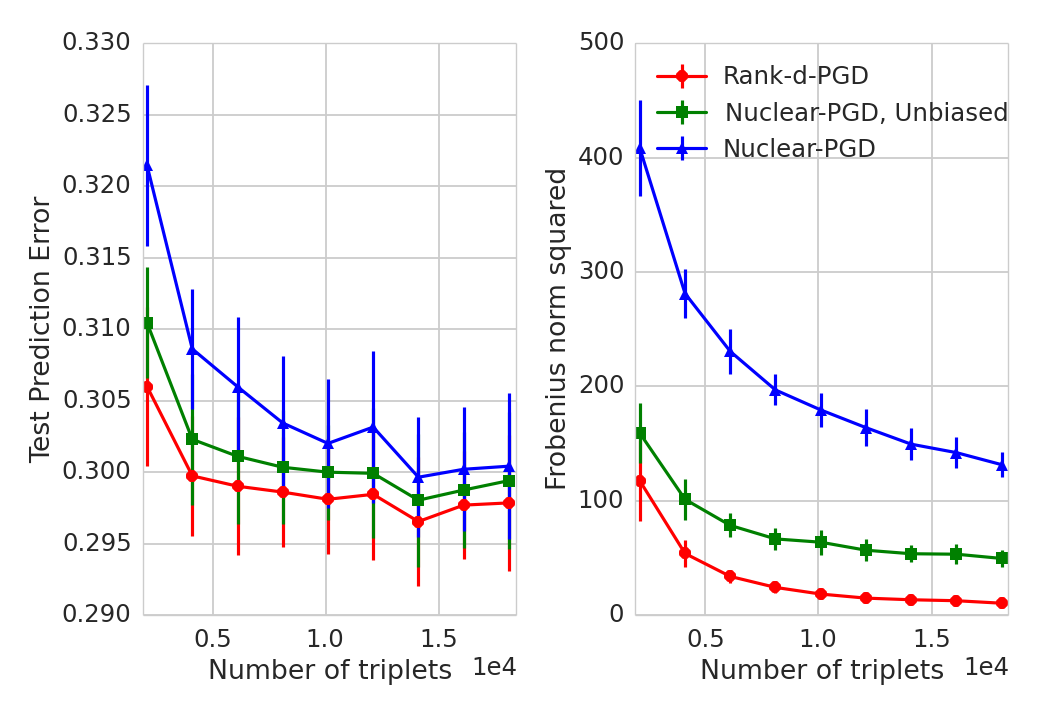} \\
    \includegraphics[width=\imwidth]{./experiments/mds_a53e3a6b627c50a5000ee5f11ae0e8a8.png} \\
    \includegraphics[width=\imwidth]{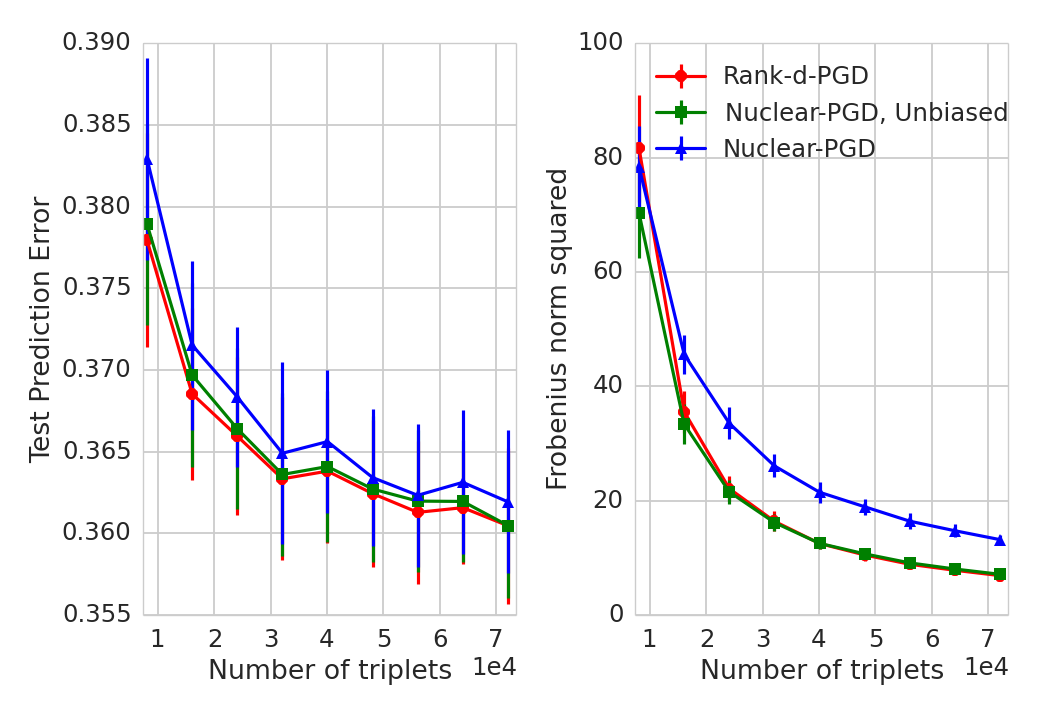} \\
    \includegraphics[width=\imwidth]{./experiments/mds_a0be90060c30c01f0957c62b04c7274e.png}
  \end{center}
  \caption{\textbf{Varying dimension} $n=64$, $d = \{1,2,4,8\}$ from top to bottom. \label{fig:varying_dimension}}
\end{minipage}
\hspace{.25cm}
\begin{minipage}{.45\textwidth}
  \begin{center}
    \includegraphics[width=\imwidth]{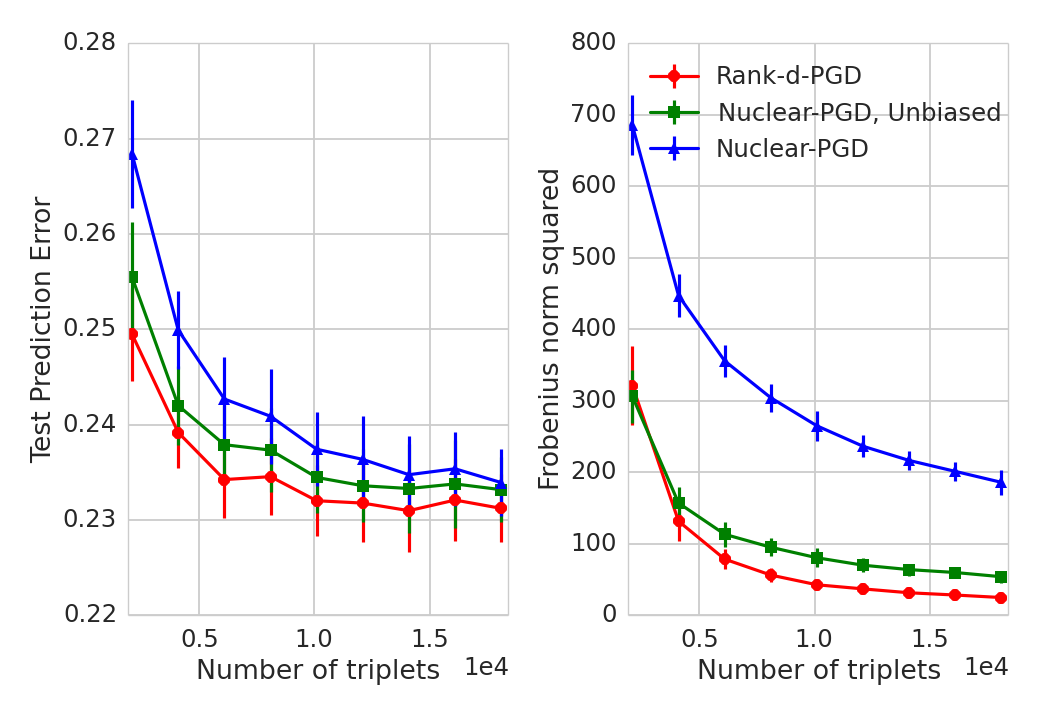} \\
    \includegraphics[width=\imwidth]{./experiments/mds_a53e3a6b627c50a5000ee5f11ae0e8a8.png} \\
    \includegraphics[width=\imwidth]{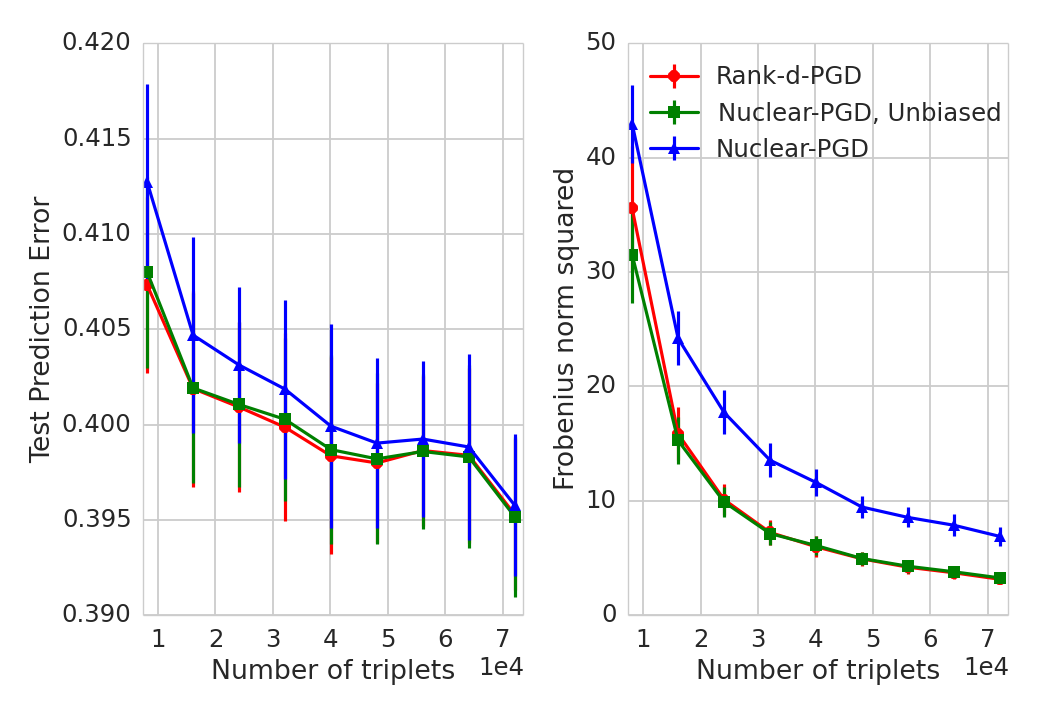} \\
    \includegraphics[width=\imwidth]{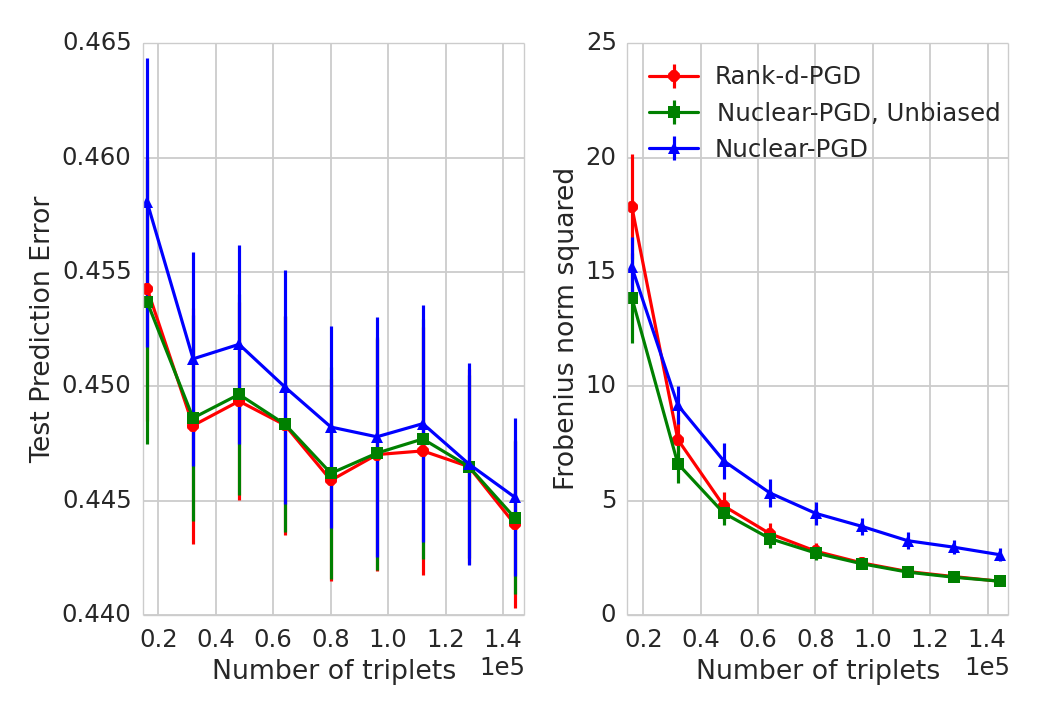}
  \end{center}
  \caption{\textbf{Varying Noise} $n=64$, $d=2$, $\alpha M$ for $\alpha = \{2,1,\tfrac{1}{2},\tfrac{1}{4}\}$ from top to bottom. \label{fig:varying_noise}}
\end{minipage}
\end{figure}

\begin{figure}[!ht]
  \begin{minipage}{.45\textwidth}
  \begin{center}
    \includegraphics[width=\imwidth]{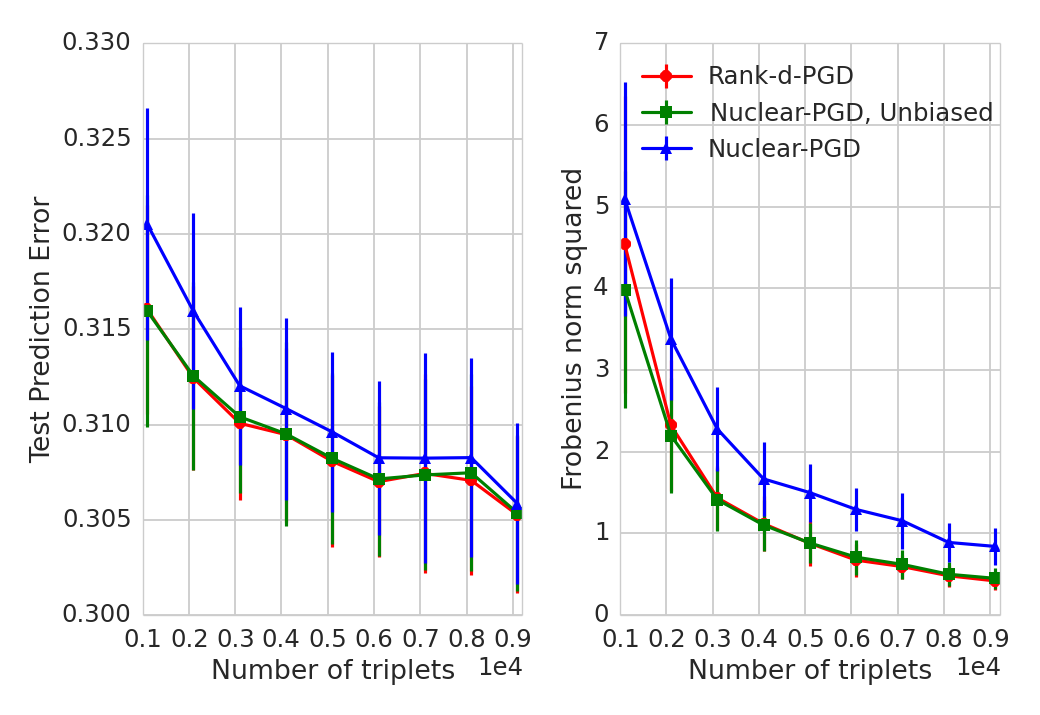} \\
    \includegraphics[width=\imwidth]{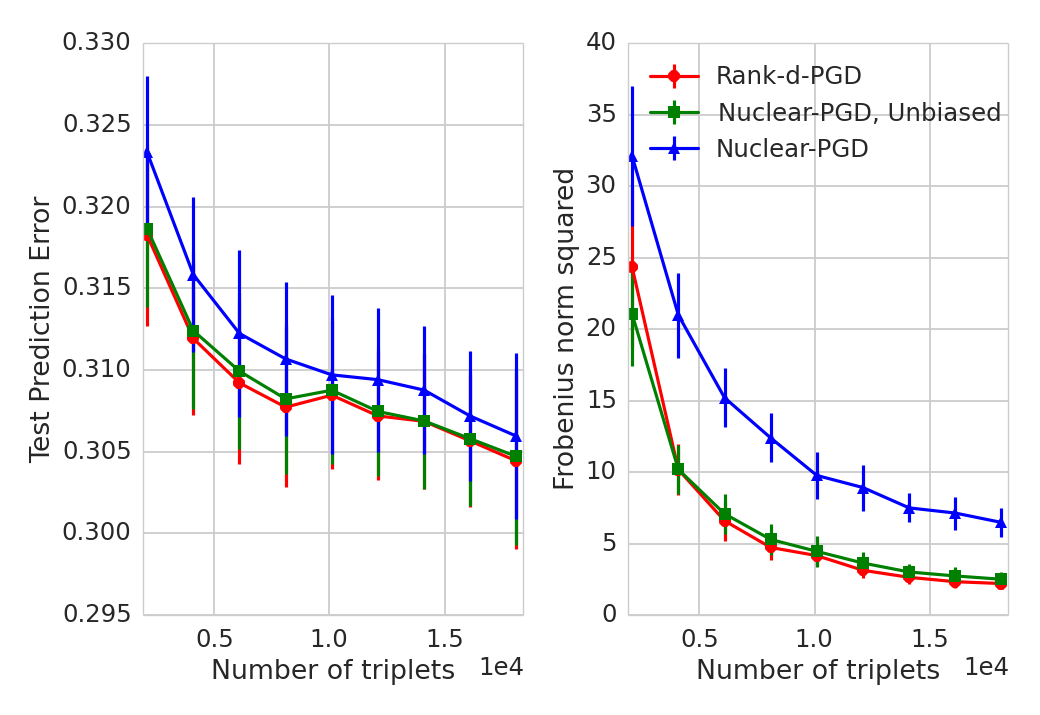} \\
    \includegraphics[width=\imwidth]{./experiments/mds_a53e3a6b627c50a5000ee5f11ae0e8a8.png} \\
    \includegraphics[width=\imwidth]{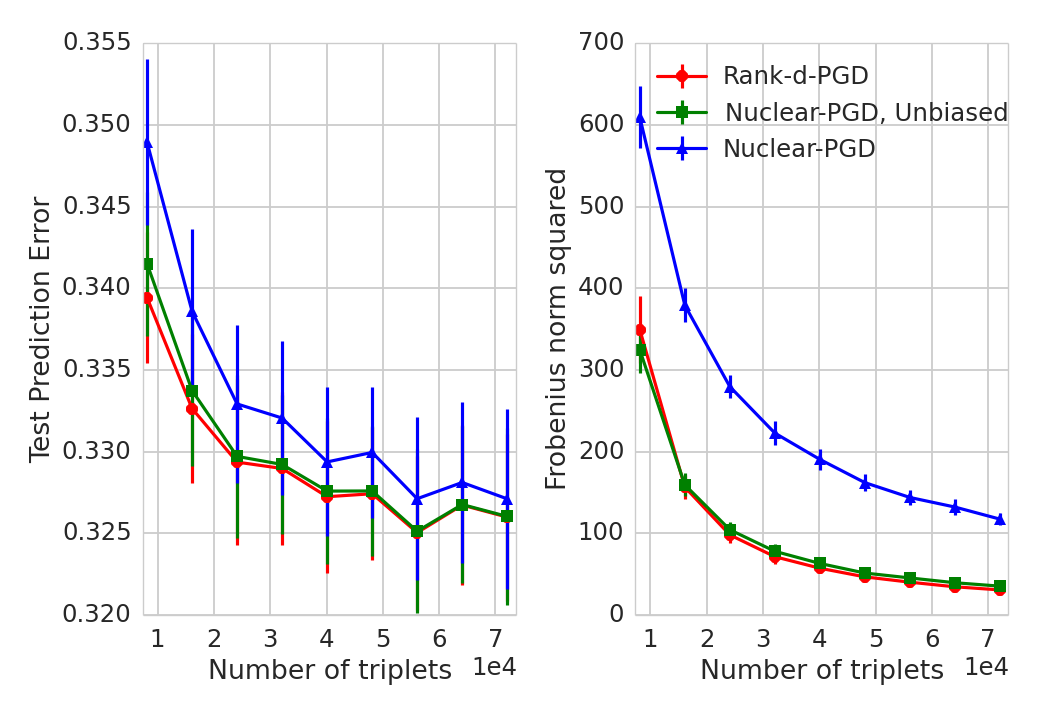} \\
  \end{center}
  \caption{\textbf{Varying $\#$ items} $n = \{16,32,64,128\}$, $d=2$ from top to bottom.}
  \end{minipage}
\end{figure}

% \include{projected_convergence}

%\end{document}